\documentclass[runningheads,a4paper]{llncs}

\usepackage{url}

\usepackage[ruled,vlined]{algorithm2e}
\usepackage{algorithmic}
\usepackage{amsmath}
\usepackage{amsfonts}
\usepackage{amssymb}
\usepackage{graphicx}
\usepackage{url}
\usepackage{subfigure}
\usepackage{epstopdf}
\usepackage{multirow}
\usepackage{subfigure}
\usepackage{ifthen}

\newcommand{\defterm}{\textbf}

\newcommand{\A}{\mathbb{A}}
\newcommand{\B}{\mathbb{B}}

\renewcommand{\a}{a}

\newcommand{\TT}{T}

\newcommand{\UT}{U}

\newcommand{\TV}{t}
\newcommand{\TTuple}[1][0.0ex]{\vec{t}\hspace{#1}}
\newcommand{\UV}{u}
\newcommand{\UTuple}[1][0.0ex]{\vec{u}\hspace{#1}}

\newcommand{\VTuple}{\vec{v}}

\newcommand{\QC}{\FG{\Lambda}}

\newcommand{\Qconj}{\Appendterm{\FG{\TT_{\grounding}} = \TV} {\QC}}

\newcommand{\FG}[2][0.0ex]{#2^{*}\hspace{#1}}

\newcommand{\Pa}[1]{\mathrm{Pa}(#1)}
\newcommand{\Ch}[1]{\mathrm{Ch}(#1)}

\newcommand{\Man}{\mathrm{M}}
\newcommand{\Woman}{\mathrm{W}}

\newcommand{\sepcup}[1][0.5ex]{\hspace{#1}\cup\hspace{#1}}
\newcommand{\Setaddterm}[2]{#1 \sepcup #2}
\newcommand{\Appendterm}[2]{#1, #2}

\newcommand{\Mrange}[1]{\ifthenelse{\equal{#1}{T}}{\TTuple_m}{\ifthenelse{\equal{#1}{U}}{\UTuple_m}{\ifthenelse{\equal{#1}{V}}{\VTuple_m}{\mbox{UNKNOWN
TERM ID}}}}}
\newcommand{\Prange}[1]{\ifthenelse{\equal{#1}{T}}{\vec{t}_{pa}}{\ifthenelse{\equal{#1}{U}}{\vec{u}_{pa}}{\ifthenelse{\equal{#1}{V}}{\vec{v}_{pa}}{\mbox{UNKNOWN
TERM ID}}}}}

\newcommand{\GroundPrange}[1]{\ifthenelse{\equal{#1}{T}}{\vec{t}_{pa,\grounding'}}{\ifthenelse{\equal{#1}{U}}{\vec{u}_{pa,\grounding'}}{\ifthenelse{\equal{#1}{V}}{\vec{v}_{pa,\grounding'}}{\mbox{UNKNOWN
TERM ID}}}}}

\newcommand{\joint}{p}

\newcommand{\cprob}[2]{\theta(#1|#2)}

\newcommand{\Gpvar}{P}
\newcommand{\Gprob}[2]{\Gpvar(#1 | #2)}
\newcommand{\Cvar}{\mathrm{n}}
\newcommand{\Fvar}{\mathrm{p}}
\newcommand{\Count}[2]{\Cvar\left[#1;#2\right]}

\newcommand{\Relevant}[1]{#1^{\mathrm{r}}}
\newcommand{\Relcount}[2]{\Relevant{\Cvar}\left[#1;#2\right]}

\newcommand{\Relfreq}[2]{\Relevant{\Fvar}\left[#1;#2\right]}

\newcommand{\parents}{\mathbf{pa}}

\newcommand{\functor}{f}

\newcommand{\term}{\tau}

\newcommand{\grounding}{\gamma}

\newcommand{\true}{\mathrm{T}}
\newcommand{\false}{\mathrm{F}}

\renewcommand{\Qconj}{\Appendterm{\FG{\TT} = \TV} {\QC}} 

\usepackage{graphicx}

\newcommand{\iid}{i.i.d.}

\title{Fast Learning of Relational Dependency Networks}

\author{Oliver Schulte, Zhensong Qian, Arthur E. Kirkpatrick\\ Xiaoqian Yin, Yan Sun 
}
\authorrunning{  Oliver Schulte, Zhensong Qian, Arthur E. Kirkpatrick et al.}
\institute{ School of Computing Science, Simon Fraser University, Canada\\
\{oschulte,zqian,ted,xiaoqian\_yin,sunyans\}@sfu.ca\\}

\begin{document}

\maketitle

\begin{abstract} 
A Relational Dependency Network (RDN) is a directed graphical model widely used for multi-relational data. These networks allow cyclic dependencies, necessary to represent relational autocorrelations. We describe an approach for learning both the RDN's structure and its parameters, given an input relational database: First learn a Bayesian network (BN), then transform the Bayesian network to an RDN. Thus fast Bayes net learning can provide fast RDN learning. The BN-to-RDN transform comprises a simple, local adjustment of the Bayes net structure and a closed-form transform of the Bayes net parameters. This method can learn an RDN for a dataset with a million tuples in minutes. We empirically compare our approach to state-of-the art RDN learning methods that use functional gradient boosting, on five benchmark datasets. Learning RDNs via BNs scales much better to large datasets than learning RDNs with boosting, and provides competitive accuracy in predictions.\end{abstract}

 \section{Introduction} \label{sec:intro} Learning graphical models is one of the main approaches to extending machine learning for relational data. 
Dependency networks (DNs) \cite{Heckerman2000} are one of the major classes of graphical generative models, together with Markov networks and Bayesian networks (BNs) \cite{Pearl1988}. We describe a new approach to learning dependency networks: first learn a Bayesian network, then convert the Bayesian network to a dependency network. 
This hybrid approach combines the advantages of learning with Bayesian networks and performing inference with relational dependency networks. 
The hybrid learning algorithm produces dependency networks for large complex databases, with up to one million records, and up to 19 predicates. The predictive accuracy of the learned dependency networks is competitive with those constructed by state-of-the-art function gradient boosting methods.
Bayesian network learning scales substantially better   to larger datasets than the boosting methods.
Our main contributions are:
\begin{enumerate}
\item A faster approach for learning relational dependency networks: first learn a Bayesian network, then convert it to a dependency network.
\item A closed-form log-linear discriminative model for computing the relational dependency network parameters from Bayesian network structure and parameters.
\end{enumerate}
  
 \section{Relational Dependency Networks and Bayesian Networks} We review the definition of dependency networks and their advantages for modelling relational data. We assume familiarity with the basic concepts of Bayesian networks \cite{Pearl1988}.
 
 \subsection{Dependency networks and Bayesian networks} Like Bayesian networks, the structure of a dependency network is defined by a graph whose nodes are random variables and whose edges are directed. Unlike Bayesian networks, a dependency network graph may contain cycles and bi-directed edges. As with Bayesian networks, the parameters of dependency networks are conditional distributions over the value of a child node given its parents. The difference lies in the characteristic independence property of dependency networks: each node is independent of {\em all} other nodes given an assignment of values to its parents, which is generally not the case for Bayesian networks.
In graphical model terms, the parents of a node in a dependency network form a Markov blanket: A minimal set of nodes such that assigning them values will make this node independent of the rest of the network. 

Consequently, a parameter in a dependency network effectively specifies the probability of a node value given an assignment of values to all other nodes. 
We therefore refer to such conditional probabilities as \defterm{Gibbs conditional probabilities}, or simply Gibbs probabilities.\footnote{In the terminology of dependency networks \cite{Heckerman2000},  Gibbs  probabilities are referred to as local probability distributions.}
Gibbs sampling can be used to derive a joint distribution from the Gibbs probability DN parameters \cite{Heckerman2000,Neville2007}. This is the counterpart to the Bayes net product formula that derives a joint distribution from the network's conditional probability parameters. 

\begin{figure}[htbp]
\begin{center}
\includegraphics[width = 0.7 \textwidth]{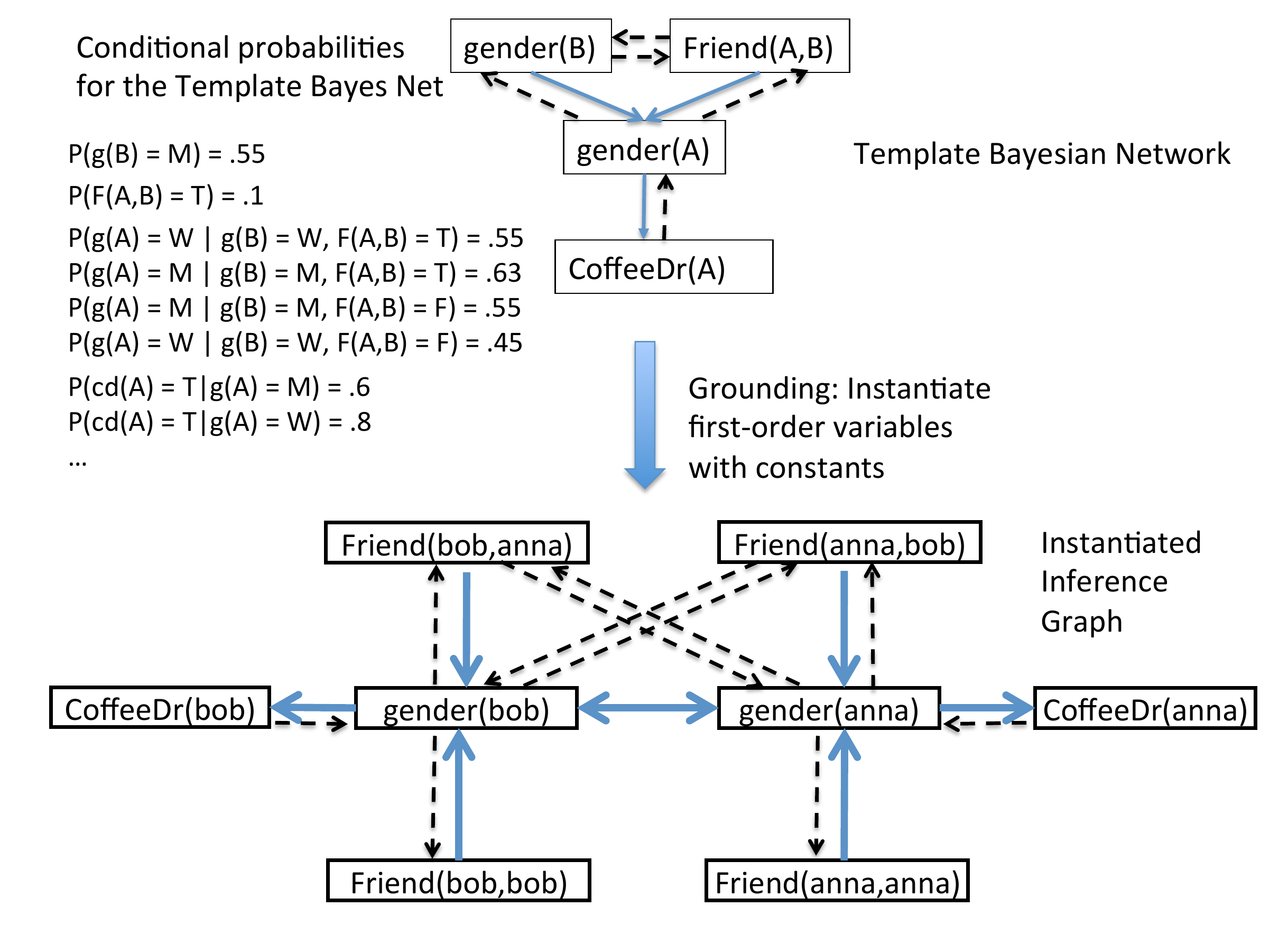}
\caption{A Bayesian/dependency template network (top) and the instantiated inference graphs (bottom). BN edges are shown as blue and solid. The BN-to-DN transformation adds the edges shown as black and dashed. Notice that grounding the BN induces a bi-directed edge between $\it{gender}(bob)$ and $\it{gender}(anna)$. \label{fig:dn}}
\end{center}
\end{figure}

\subsection{Relational Dependency Networks} We use  functor-based notation for graphical-relational models \cite{Poole2003}. A functor is a function or predicate symbol. Each functor has a set of values (constants) called the \textbf{domain} of the functor. In this paper we consider only functors with finite domains. A \textbf{Parametrized Random Variable} (PRV) is of the form $f(\term_{1},\ldots,\term_{k})$ where $\functor$ is a functor 
and each $\term_{i}$ is a first-order variable or a constant.
 A Parametrized Bayesian Network structure is a directed acyclic graph whose nodes are PRVs. A \textbf{relational dependency network structure} (RDN) is a directed graph whose nodes are PRVs.
RDNs extend dependency networks for relational data by using knowledge-based model construction \cite{Neville2007}:
 The first-order variables in a template RDN graph are instantiated for a specific domain of individuals to produce an {\em  instantiated} or {\em ground} propositional DN graph, the \defterm{inference graph}. Figure~\ref{fig:dn} gives a dependency network template and its grounded inference graph. An example Gibbs probability distribution for the inference graph (abbreviating functors to their first letter) is

$$P(\it{g(anna)}|\it{g(bob)}, \it{CD(anna)}, \it{F(anna,bob)},\it{F(bob,anna)},\it{F(anna,anna)}).$$

\noindent Both the structure and the parameter space of RDN models offer advantages for relational data \cite{Neville2007,Natarajan2012}: (1) Dependency network structures are well-adapted for relational data because they allow cyclic dependencies, so grounding a dependency network template is guaranteed to produce a valid dependency network.
(2) Relational prediction  requires aggregating information from different linked individuals \cite{Natarajan2008}. 
In a dependency network parameter, the aggregation encompasses the entire Markov blanket of a target node, whereas for Bayesian network parameters, the aggregation encompasses only its parents.

\section{Learning Relational Dependency Networks via Bayesian Networks}
Our algorithm for rapidly learning relational dependency networks
begins with any relational learning algorithm for Bayesian networks. We then apply a simple, fast transformation of the resulting Bayesian network to a relational dependency template. Finally we apply a closed-form computation to derive the dependency network parameters from the Bayesian structure and parameters. Figure~\ref{fig:bn-flow} shows the program flow. 

Converting a Bayesian network structure to a dependency network structure is simple: for each node, add an edge pointing to the node from each member of its BN Markov blanket~\cite{Heckerman2000}.  The result contains  bidirectional links between each node, its children, and its co-parents (nodes that share a child with this one). 
This is equivalent to the standard moralization  method for converting a BN to an undirected model \cite{Domingos2009}, except that the dependency network contains bi-directed edges instead of undirected edges. Bidirected edges have the advantage that they permit  assignment of different parameters to each direction, whereas undirected edges have only one parameter.
 
Converting Bayesian network parameters to dependency network parameters is simple for propositional \iid{} data: solve for the Gibbs conditional probabilities given Bayesian network parameters. The propositional result is as follows. A \defterm{family} comprises a node and its parents. A \defterm{family configuration} specifies a value for a child node and each of its parents. For example in the Bayesian network of Figure~\ref{fig:dn}, a family configuration is 
$$\it{gender}(\A) = \Man, \it{Friend}(\A,\B) = \true, \it{gender}(\B) = \Man.$$
For propositional data, an assignment of values to the Markov blanket of a target node assigns a unique configuration for each family whose child is the target node or one of its children. Hence the Markov blanket induces a {\em unique} log-conditional probability for each such family configuration. The probability of a target node value given an assignment of values to the Markov blanket is then proportional to the exponentiated sum of these log-conditional probabilites \cite[Ch.14.5.2]{Russell2010}. 

With relational data, different family configurations such as the one above can be simultaneously instantiated, multiple times.  We adapt the propositional log-linear equation for relational data by replacing the unique log-conditional probability with the {\em expected} log-conditional probability that results from selecting an instantiation of the family configuration uniformly at random. The probability of a target node value given an assignment of values to the Markov blanket is then proportional to the exponentiated sum of the expected log-conditional probabilites. 
We describe the resulting closed-form equation in the next section.

\begin{figure}[t]

\begin{center}
\includegraphics[width=0.7\textwidth]{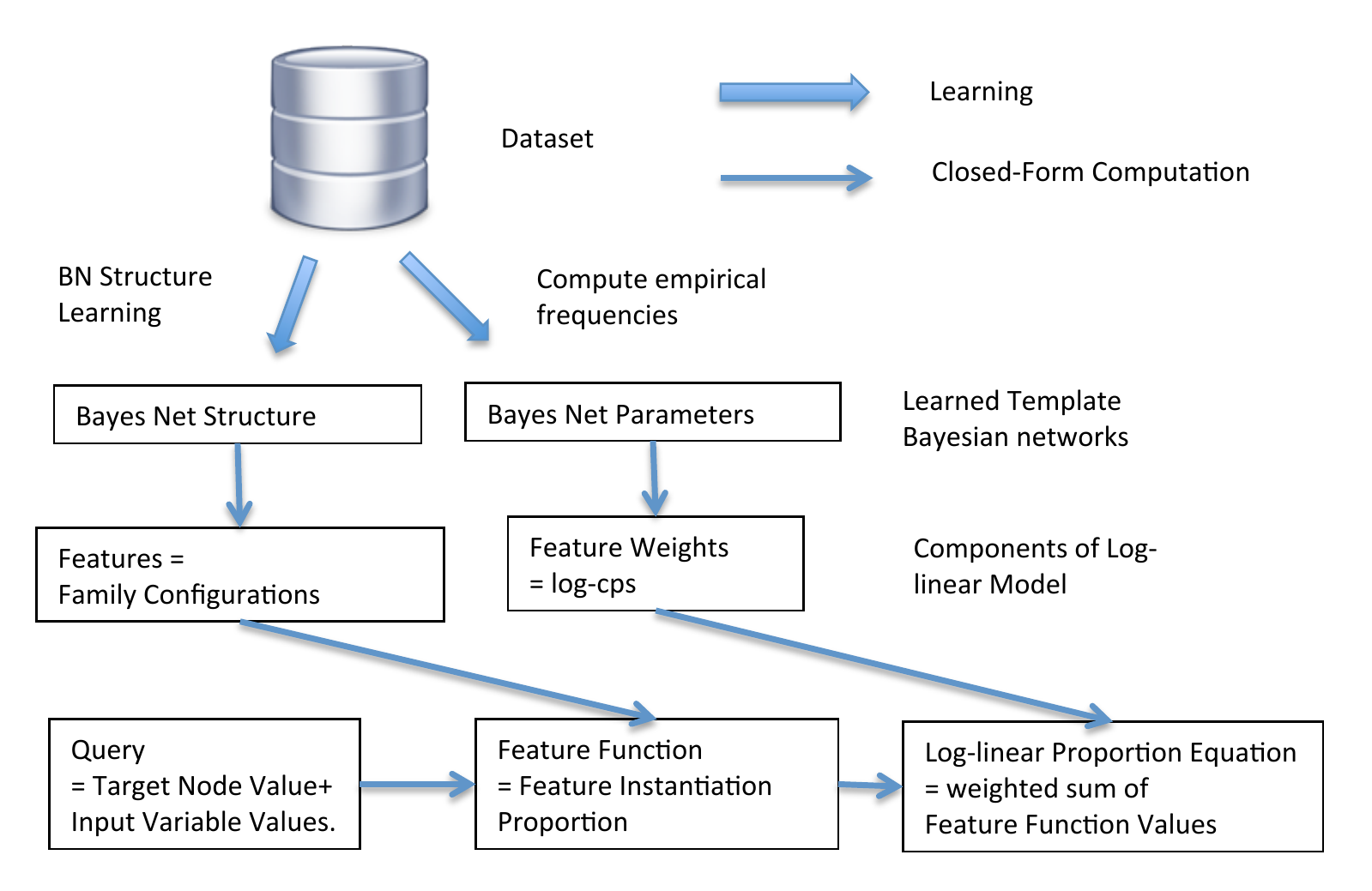}
\caption{The program flow for computing Gibbs probabilities from a template Bayesian network. Features and weights are computed from the Bayes net. Feature function values are computed for each query. \label{fig:bn-flow}}
\end{center}

\end{figure}

\section{The Log-linear Proportion Equation} 
\label{sec:theequation}

We propose a log-linear equation, the \defterm{log-linear proportion equation}, for computing a Gibbs conditional probability for a ground target node, $\FG{\TT}$, given (i) a target value $\TV$ for the target node, (ii) a complete set of values $\QC$  for all ground terms other than the target node, and (iii) a template Bayesian network. The template structure is represented by functions that return the set of parent nodes of $\UT$, $\Pa{\UT}$, and the set of child nodes of $\UT$, $\Ch{\UT}$. The parameters of the template are
represented by the conditional probabilities of a node $\UT$ having a value $\UV$ conditional on the values of its parents, $\cprob{\UT = \UV}{\Pa{\UT} = \Prange{\UT}}$. A grounding $\grounding$ substitutes a constant for each member of a list of first-order variables. A grounding is therefore equivalent to an equality constraint $\{\A_{1} = \a_{1},\ldots, A_{k} = \a_{k}\}$. Applying a grounding to a template node defines a fully ground target node. For instance, we may have $\it{gender}(\A) \{\A = sam\} = \it{gender}(sam)$.  These are combined in the following log-linear equation:

\begin{definition}[The Log-Linear Proportion Equation]\label{def:log-diff-freq-eq}
\begin{eqnarray*}
  \Gprob{\FG{{\TT}} = \TV} {\QC} &\propto &  \\
 \sum_{\UT} \sum_{\UV,\Prange{\UT}}   
\qquad \left[ \ln \cprob{\UT = \UV}{\Pa{\UT} = \Prange{\UT}} \right] &
    \cdot &
    \Relfreq{\Appendterm{\grounding;\UT  = \UV} {\Pa{\UT} = \Prange{\UT}}} {\Qconj}
\end{eqnarray*}
where 
\begin{eqnarray*}
\UT &\mathrm{varies\ over} & \Setaddterm{\{\TT\}} {\Ch{\TT}}, \\
\mbox{the singleton value} \ \UV & \mathrm{varies\ over} & \mbox{the range of}\  \UT,\\
\mbox{the vector of values} \ \Prange{\UT} & \mathrm{varies\ over} & \mbox{the product of the ranges of} \ \UT's\ \mbox{parents}, \\
\FG{\TT} = \TT \grounding&\mathrm{is} & \mbox{ is the target node grounding of template node }  \TT, \mathrm{and} \\
\Relevant{\Fvar} &\mathrm{is} & \mbox{the proportion feature function}.
\end{eqnarray*}
\end{definition}

The feature function $\Relevant{\Fvar}$ specifies the proportion of instantiations that satisfy a given family configuration, relative to all family configurations with positive links only. This proportion is computed as follows. 

\begin{enumerate}
\item For a given family configuration $(\Appendterm{\UT  = \UV} {\Pa{\UT} = \Prange{\UT}})$, let the \textbf{family  count} $$\Count{\Appendterm{\grounding;\UT  = \UV} {\Pa{\UT} = \Prange{\UT}}} {\Qconj}$$ be the number of instantiations that (a) satisfy the family configuration and the ground node values specified by $\Qconj$, and (b) are consistent with the equality constraint defined by $\grounding$. This notation is consistent with the parfactor notation of \cite{Poole2003}. 
\item The \textbf{relevant family count} $n^{r}$ is 0 if the family configuration contains a false relationship (other than the target node), else equals the feature count.
\item The \textbf{family proportion} is the relevant family count, divided by the total sum of all relevant family counts for the given family. In symbols:

\begin{equation} \notag
 \Relfreq{\Appendterm{\grounding;\UT  = \UV} {\Pa{\UT} = \Prange{\UT}}} {\Qconj} = \frac{\Relcount{\Appendterm{\grounding;\UT  = \UV} {\Pa{\UT} = \Prange{\UT}}} {\Qconj}}{\sum_{\UV',\Prange{\UT}'}\Relcount{\Appendterm{\grounding;\UT  = \UV'} {\Pa{\UT} = \Prange{\UT}'}} {\Qconj}}
\end{equation}
\end{enumerate}

It is common in statistical-relational models to restrict predictors to existing relationships only \cite{Getoor2007c,Russell2010}. The inner sum of Formula~\ref{def:log-diff-freq-eq} computes the expected log-conditional probability for a family with child node $\UT$, when we randomly select a relevant grounding of the first-order variables in the family. 

Definition~\ref{def:log-diff-freq-eq} has the form of a log-linear model \cite{Sutton2007}: The features of the model are the family configurations $(\Appendterm{\UT  = \UV} {\Pa{\UT} = \Prange{\UT}})$ 
where the child node is either the target node or one of its chldren. The feature weights are the log-conditional BN probabilities defined for the family configuration. The input variables are the values of the ground nodes other than the target nodes, specified by the conjunction $\QC$. The family count specifies how many times the feature is instantiated in the input variables (plus the target node value). The family proportion is the feature function, which maps a feature to a real value given the input variables. 
Proportions have the desirable consequence that all feature functions are normalized to the [0,1] range. 
\paragraph{Example.}
Table~\ref{table:log-diff-example} illustrates the computation of our log-linear model for predicting the gender of a new test instance ($sam$).
\begin{table}
\caption{Applying the log-linear proportion equation with the Bayesian network of Figure~\ref{fig:dn} to compute $\Gprob{\it{gender}(sam) = \Woman} {\QC}$ and $\Gprob{\it{gender}(sam) = \Man} {\QC}$. Each row represents a feature/family configuration. For the sake of the example we suppose that the conjunction $\QC$ specifies that Sam is a coffee drinker, has 60 male friends, and 40 female friends. $CP$ refers to the conditional probability BN parameter of Figure~\ref{fig:dn}. For the feature weights $w \equiv \ln(CP)$.
\label{table:log-diff-example}}
\centering
\begin{tabular}{l@{\hspace{.1in}}l@{\hspace{.1in}}r@{\hspace{.1in}}r@{\hspace{.1in}}r@{\hspace{.1in}}r}
{\setlength{\tabcolsep}{0pt}\begin{tabular}{l}Child \\Value $\UV$ \end{tabular}}&Parent State $\Prange{\UT}$&CP 
&{\setlength{\tabcolsep}{0pt}\begin{tabular}{c} $w$ \end{tabular}}&$\Relevant{\Fvar}$&{\setlength{\tabcolsep}{0pt}\begin{tabular}{c} $w \times \Relevant{\Fvar}$ \end{tabular}} \\\hline
$g(sam) = \Woman$&{\setlength{\tabcolsep}{0pt}\begin{tabular}{l}$ g(B) = \Woman,$\\ $F(sam,B) = \true$\end{tabular}}&$0.55$&$-0.60$&$0.4$&$-0.24$ \\
$g(sam) = \Woman$&{\setlength{\tabcolsep}{0pt}\begin{tabular}{l}$ g(B) = \Man,$\\ $ F(sam,B) = \true$\end{tabular}}&$0.37$&$-0.99$&$0.6$&$-0.60$ \\
$CD(sam) = \true$&{\setlength{\tabcolsep}{0pt}\begin{tabular}{l}$ g(sam) = \Woman$\end{tabular}}&$0.80$&$-0.22$&$1.0$&$-0.22$ \\
$CD(sam) = \false$&{\setlength{\tabcolsep}{0pt}\begin{tabular}{l}$g(sam) = \Woman$ \end{tabular}}&$0.20$&$-1.61$&$0.0$&$0.00$ \\\hline
\multicolumn{5}{l}{Sum ($\exp(Sum) \propto\Gprob{gender(sam)=\mathrm{\Woman}}{\QC}$)}&$-1.06$ \\\hline
$g(sam) = \Man$&{\setlength{\tabcolsep}{0pt}\begin{tabular}{l}$ g(B) = \Woman,$\\ $F(sam,B) = \true$\end{tabular}}&$0.45$&$-0.80$&$0.4$&$-0.32$ \\
$g(sam) = \Man$&{\setlength{\tabcolsep}{0pt}\begin{tabular}{l}$ g(B) = \Man,$\\ $ F(sam,B) = \true$\end{tabular}}&$0.63$&$-0.46$&$0.6$&$-0.28$ \\
$CD(sam) = \true$&{\setlength{\tabcolsep}{0pt}\begin{tabular}{l}$ g(sam) = \Man$\end{tabular}}&$0.60$&$-0.51$&$1.0$&$-0.51$ \\
$CD(sam) = \false$&{\setlength{\tabcolsep}{0pt}\begin{tabular}{l}$ g(sam) = \Man$\end{tabular}}&$0.40$&$-0.92$&$0.0$&$0.00$ \\\hline
\multicolumn{5}{l}{Sum ($\exp(Sum) \propto\Gprob{gender(sam)=\mathrm{\Man}}{\QC}$)}&$-1.11$ \\\hline

\end{tabular}
\end{table}

\paragraph{Estimating Bayes net parameters.}
The Bayesian network parameters can be estimated by applying the maximum likelihood principle, which entails using the empirical conditional frequencies observed in an input relational database \cite{Schulte2011,Schulte2014}. 
 Although there is theoretical justification for using the empirical frequencies, the ultimate test is whether the method can achieve comparable accuracy and greater speed than prior methods of computing relational dependency networks. In the next section, we empirically compare these methods.

\section{Empirical Comparison with Functional Gradient Boosting}\label{sec:empirical-comparison}

The next section describes experiments that compare learning RDNs via Bayesian networks with functional gradient methods for learning relational dependency networks. Boosting methods follow the traditional approach to learning dependency networks, which is to learn a collection of separate discriminative models, one for each node in the network \cite{Heckerman2000}. Functional gradient boosting has been shown to perform well on small datasets previously \cite{Khot2011,Natarajan2012}; our experiments provide new new tests of this method on medium to large datasets. 

\subsection{Experimental Conditions and Metrics}\label{sec:conditions}

All experiments were done on with 8GB of RAM and a single Intel Core 2 QUAD Processor Q6700 with a clock speed of 2.66GHz (there is no hyper-threading on this chip). The operating system was Linux Centos 2.6.32. Code was written in Java, JRE 1.7.0. All code and datasets are available~\cite{bib:jbnsite}. 

\subsubsection{Datasets}
We used 
5 benchmark real-world databases. For more details please see the references in \cite{Schulte2012}. Summary statistics appear in Table~\ref{table:learning-times}.

\begin{description}

\item[MovieLens Databases] MovieLens is a  commonly-used rating dataset\footnote{www.grouplens.org}. We added more related attribute information about the actors, directors and movies from the Internet Movie Database (IMDB)\footnote{www.imdb.com, July 2013}.
It contains two entity sets, Users and Movies. For each user and movie that appears in the database, all available ratings are included. MovieLens(1M) contains 1M ratings, 3,883 Movies, and 6,039 Users. MovieLens(0.1M) contains about 0.1M ratings, 1,682 Movies, and 941 Users. We did not use the binary genre predicates because they are easily learned with exclusion rules.

\item[Mutagenesis Database] This dataset is widely used in Inductive Logic Programming research. 
It contains information on Atoms, Molecules, and Bonds between them. We use the discretization of \cite{Schulte2012}.

\item[Hepatitis Database] This data is a modified version of the PKDD02 Discovery Challenge database. 
The database contains information on the laboratory examinations of hepatitis B and C infected patients. 

\item[Mondial Database] 
This dataset contains data from multiple geographical web data sources.

\item[UW-CSE Database] This dataset lists facts about the Department of Computer Science and Engineering at the University of Washington, such as entities (e.g., $Person$, $Course$) and the relationships (i.e. $AdvisedBy$, $TaughtBy$).

\end{description}

\subsubsection{Methods Compared} Functional gradient boosting is a state-of-the-art method for applying discriminative learning to build a generative graphical model. The local discriminative models are  ensembles of relational regression trees \cite{Khot2011}. Our experiments used the Boostr implementation of relational gradient boosting \cite{Khot2013}. The current implementation does not support multi-class boosting, so following previous experiments \cite{Khot2011}, we limited our comparison to {\em binary predicates}, i.e., functors that can take on only two possible values (e.g., $\it{AdvisedBy}$). We compared the following three learning methods.

\begin{description}
\item[RDN\_Bayes] Learns a Bayesian network, then converts it to a relational dependency network as described above.
\item[RDN\_Boost] The state-of-the-art gradient boosting method designed for learning RDNs. Information from ground nodes that are linked to the target node is aggregated with functions $count, max, average$ and existential quantification \cite{Natarajan2012}.
\item[MLN\_Boost] The state-of-the-art gradient boosting method designed for learning Markov Logic Networks. It takes as input a  list of target predicates for analysis. To construct an RDN, we provide each binary predicate as a single target predicate in turn. Information from ground nodes that are linked to the target node is aggregated with a log-linear model derived from Markov Logic Networks.
\end{description}

We followed the Boostr instructions for creating the background .bk file and used the default settings. We experimented with alternative settings but they did not improve the performance of the boosting methods.

To obtain the BN structure for RDN\_Bayes, the learn-and-join algorithm~\cite{Schulte2012} was applied to each benchmark database. The BN parameters were computed from the empirical conditional frequencies in the database using previously-published algorithms~\cite{Schulte2014}.

\subsubsection{Prediction Metrics}
We follow \cite{Khot2011} and evaluate the algorithms using conditional log likelihood (CLL) and AUC-PR (Area Under Precision-Recall Curve). AUC-PR is appropriate when the target predicates features a skewed distribution as is typically the case with relationship predicates.
For each fact $\FG{\TT} = \TV$ in the test dataset, we evaluate the accuracy of the predicted Gibbs probability $\Gprob{\FG{\TT} = \TV} {\QC}$, where $\QC$ is a complete conjunction for all ground terms other than $\FG{\TT}$. Thus $\QC$ represents the values of the input variables as specified by the test dataset.
CLL is the average of the logarithm of the Gibbs probability for each ground truth fact in the test dataset. For the gradient boosting method, we used the AUC-PR and likelihood scoring routines included in Boostr.

Both metrics are reported as averages over all binary predicates. The learning methods were evaluated using 5-fold cross-validation. Each database was split into 5 folds by randomly selecting entities from each entity table, and restricting the relationship tuples in each fold to those involving only the selected entities  (i.e., subgraph sampling~\cite{Schulte2012}). The models were trained on 4 of the 5 folds, then tested on the remaining one. All results are averages from 5-fold cross validation, over all descriptive attributes in the database.

\subsection{Results} 

Table~\ref{table:learning-times} shows learning times for the different methods. For the boosting method, we added together the learning times for each target predicate. The total learning times are not directly comparable because Bayes net learning simultaneously learns a joint model for all predicates. We therefore report total learning time divided by the number of all predicates for RDN\_Bayes, and total learning time divided by the number of binary predicates for the boosting methods. The numbers of predicates are given in the second column.
\begin{table}[htbp]
  \centering
  \caption{Learning Time (Sec) Per Predicate}
    \begin{tabular}{|l|p{2cm}|r|r|r|r|}
\hline
     Dataset & all predicates / binary predicates & \# tuples & RDN\_Bayes & RDN\_Boost & MLN\_Boost \\ \hline
    UW    & 14/4  & 612   & 0.74$\pm$0.05 & 14.57$\pm$0.39 & 19.27$\pm$0.77  \\
    Mondial & 18/4  & 870   & 101.53$\pm$6.90 & 27.21$\pm$0.98 & 41.97$\pm$1.03 \\
    Hepatitis & 19/7  & 11,316 & 285.71$\pm$20.94 & 250.61$\pm$5.32 & 229.73$\pm$2.04  \\
    Mutagenesis & 11/6  & 24,326 & 0.70$\pm$0.02 & 117.70$\pm$6.37 & 48.65$\pm$1.34 \\ 
    MovieLens(0.1M) & 7/2   & 83,402 & 1.11$\pm$0.08 & 2638.71$\pm$272.78 &  1866.605$\pm$112.54\\
    MovieLens(1M) & 7/2   & 1,010,051 & 1.12$\pm$0.10 & $>$24 hours & $>$24 hours \\ \hline
  
    \end{tabular}%
  \label{table:learning-times}%
\end{table}%

Table~\ref{table:learning-times} shows that RDN\_Bayes scales very well with the number of data tuples: even the large MovieLens dataset with 1M records can be analyzed in seconds. Learning separate discriminative models  scales well with the number of predicates, which is consistent with findings from  propositional learning \cite{Heckerman2000}. Bayes net learning slows down more as more predicates are included, since it learns a joint model over all predicates simultaneously. However, the learning time remains feasible (see also \cite{Schulte2012}). Bayesian network learning scales well in the number of data points because it provides closed-form parameter estimation and hence closed-form model scoring. Unlike propositional iid data, relational data are represented in multiple tables, so model evaluation requires expensive combining of information from different tables \cite{Neville2007}. Compared to learning separate discrimative models, Bayesian network explores a more complex model space, but  model evaluation is much faster. 
\begin{table}[htbp]
 \centering
  \caption{Average Conditional Log-Likelihood}
    \begin{tabular}{|r|r|r|r|r|r|} \hline
    \textbf{CLL} & UW    & Mondial  & Hepatitis & Mutagenesis  & MovieLens(0.1M) \\ \hline
   RDN\_Boost & -0.29$\pm$0.02 & -0.48$\pm$0.03 & -0.51$\pm$0.00 & -0.43$\pm$0.02 & -0.58$\pm$0.05 \\
    MLN\_Boost & -0.16$\pm$0.01 & -0.40$\pm$0.05 & -0.52$\pm$0.00 & -0.24$\pm$0.02 & -0.38$\pm$0.06 \\
    RDN\_Bayes & \textbf{-0.01$\pm$0.00} & \textbf{-0.25$\pm$0.06} & \textbf{-0.39$\pm$0.10} & \textbf{-0.22$\pm$0.07} & \textbf{-0.30$\pm$0.02} \\ \hline
    \end{tabular}%
  \label{table:cll}%

 \centering 
 \vspace{0.1cm}
 \caption{Average Area Under Precision-Recall Curve}
    \begin{tabular}{|r|r|r|r|r|r|} \hline
    \textbf{AUC-PR} & UW    & Mondial  & Hepatitis & Mutagenesis  & MovieLens(0.1M) \\ \hline
    RDN\_Boost & 0.32$\pm$0.01 & 0.27$\pm$0.01 & \textbf{0.71$\pm$0.02} & 0.63$\pm$0.02 & 0.52$\pm$0.03 \\
    MLN\_Boost & 0.52$\pm$0.01 & 0.44$\pm$0.05 & \textbf{0.71$\pm$0.02} & \textbf{0.83$\pm$0.05} & 0.52$\pm$0.05 \\
    RDN\_Bayes & \textbf{0.89$\pm$0.00} & \textbf{0.79$\pm$0.07} & 0.55$\pm$0.11 & 0.50$\pm$0.10 & \textbf{0.65$\pm$0.02} \\ \hline
    \end{tabular}%
  \label{table:AUC}%
\end{table}%

Tables~\ref{table:cll} and~\ref{table:AUC} show results for predictive accuracy. Our system resources did not suffice for evaluating the metrics on MovieLens(1M).  In terms of the likelihood assigned to the ground truth predicate value, the Bayes net method outperforms both boosting methods on all datasets (Table~\ref{table:cll}). In terms of the precision-recall curve, the Bayes net method performs substantially better than both on three datasets, and substantially worse on the two others (Table~\ref{table:AUC}). This is a satisfactory performance because boosting is a powerful method for achieving accurate predictions, and was applied to each target predicate individually to produce a tailored discriminative model. Bayesian network learning simultaneously constructed a joint model for all predicates, and used simple maximum likelihood estimation for parameter values.
 Our overall conclusion is that \emph{Bayes net learning scales much better to large datasets, and provides competitive accuracy in predictions.} 

In addition to scalability, Bayesian networks offer two more advantages. First, learning easily extends to attributes with more than two possible values. Second, the parameters and the predictions derived from them are easily interpretable. The ensemble of regression trees is more difficult to interpret, as the inventors of the boosting method noted  \cite{Natarajan2012}.

\section{Related Work}

Dependency networks were introduced in \cite{Heckerman2000} and relational dependency networks in \cite{Neville2007}. 
Heckerman {\em et al.} compare Bayesian, Markov and dependency networks for nonrelational data. 

\emph{Bayesian networks.} There are several proposals for defining directed relational template models, based on graphs with directed edges or rules in clausal format \cite{Kersting2007,Getoor2007c}. Defining the probability of a child node conditional on multiple instantiations of a parent set requires the addition of combining rules \cite{Kersting2007} or aggregation functions \cite{Getoor2007c}. 
Combining rules such as the arithmetic mean~\cite{Natarajan2008} combine global parameters with a local scaling factor, as does our log-linear model. In terms of combining rules,  our model uses the {\em geometric mean} rather than the arithmetic mean.\footnote{The geometric mean of a list of numbers $x_{1},\ldots,x_{n}$ is $(\prod_{i} x_{i})^{1/n}$. The logarithm of the geometric mean is therefore $1/n \sum_{i} \ln x_{i}$. Thus geometric mean = exp(average (logs)).} To our knowledge, the geometric mean has not been used before as a combining rule for relational data.  

\emph{Markov Networks.} Markov Logic Networks (MLNs) provide a logical template language for undirected graphical models. 
Richardson and Domingos propose transforming a Bayesian network to a Markov Logic network using moralization, with log-conditional probabilities as weights \cite{Domingos2009}. 
This is also the standard BN-to-MLN transformation recommended by the Alchemy system \cite{bib:bayes-convert}. A discriminative model can be derived from any MLN \cite{Domingos2009}.  The structure transformation was used in previous work \cite{Schulte2012}, where MLN parameters were learned, not computed in closed-form from BN parameters. The Gibbs conditional probabilities derived from an MLN obtained from converting a Bayesian network are the same as those defined by our log-linear Formula~\ref{def:log-diff-freq-eq}, {\em if} counts replace proportions as feature functions \cite{Schulte2011}. There is no MLN whose discriminative model is equivalent to our log-linear equation with  proportions as feature functions.\footnote{Disclaimer: A preliminary version of this paper was presented at the StarAI 2012 workshop, with no archival proceedings. We are indebted to workshop reviewers and participants for helpful comments.}

\section{Conclusion and Future Work} 
\label{sec:conclusion}
Relational dependency networks offer important advantages for modelling relational data. We proposed a novel approach to learning dependency networks: first learn a Bayesian network, then perform a closed-form transformation of the Bayesian network to a dependency network. The key question is how to transform BN parameters to DN parameters. We introduced a new relational adapation of the standard BN log-linear equation for the probability of a target node conditional on an assigment of values to its Markov blanket. The new log-linear equation uses a sum of expected values of BN log-conditional probabilities, with respect to a random instantiation of first-order variables. This is equivalent to using feature instantiation proportions as feature functions. We compared our approach to state-of-the-art functional gradient boosting methods  on five benchmark datasets. Learning RDNs via BNs scales much better to large datasets than with boosting, and provides competitive accuracy in predictions.

Learning a collection of discriminative models and learning a Bayesian network learning are two very different approaches to constructing dependency networks, each with strengths and weaknesses. There are various options for hybrid approaches that combine the strengths of both. (1) Fast Bayesian network learning methods can be used to select features. Discriminative learning methods should work  faster restricted to the BN Markov blanket of a target node. (2) The Bayesian network can provide an initial dependency network structure. Gradient boosting can then be used to fine-tune a discriminative model of a child node given parent nodes, replacing a flat conditional probability table. In sum, learning relational dependency networks via Bayesian networks is a novel approach that offers promising advantages for  interpretability and scalability.

\section*{Acknowledgements} 
This work was supported by Discovery Grants to Oliver Schulte from the Natural Science and Engineering Council of Canada. Zhensong Qian was supported by a grant from the China Scholarship Council. 

\section*{Appendix: Proof of Consistency Characterization} 
We show that for a given template BN, there are two ground target nodes and query conjunction $\QC$ such that the conditional distributions of the ground target nodes given $\QC$ do not agree with any joint distribution over the ground target nodes given $\QC$. We begin by establishing some properties of the template BN and the query conjunction that are needed in the second part of the proof. The second part proves the inconsistency by showing that consistency entails a constraint that is violated by the template BN for the constructed query conjunction $\QC$.

\subsection{Properties of the template BN and the input query $\QC$} The inconsistency of the BN networks arises when a parent and a child ground node have different relevant family counts. The next lemma shows that this is possible exactly when the template BN is properly relational, meaning it relates parents and children from different populations.

\begin{lemma} \label{lemma:grounding} The following conditions are equivalent for a template edge $\TT_{1} \rightarrow \TT_{2}$.
\begin{enumerate}
\item The parent and child do not contain the same population variables.
\item It is possible to find a grounding $\grounding$ for both parent and child, and an assignment $\QC$ to all other nodes, such that the relevant family count for the $\TT_{2}$ family differs for $\FG{\TT_{1}} = \grounding \TT_{1}$ 
and $\FG{\TT_{2}} = \grounding \TT_{2}$.
\end{enumerate}
\end{lemma}

\begin{proof}
If the parent and child contain the same population variables, then there is a 1-1 correspondence between groundings of the child and groundings of the parents. Hence the count of relevant family groundings is the same for each, no matter how parents and child are instantiated. If the parent and child do not contain the same population variables, suppose without loss of generality that the child contains a population variable $\A$ not contained in the parent. Choose a common grounding $\grounding$ for the parents and child node. For the ground child node, $\grounding \TT_{2}$, let $\grounding$ be the only family grounding that is relevant, so the relevant count is 1. For the  ground parent node, there is at least one other grounding of the child node $\TT_{2}'$ different from $\grounding \TT_{2}$ since $\TT_{2}$ contains another population variables. Thus it is possible to add another relevant family grounding for $\grounding \TT_{1}$, which means that the relevant count is at least 2.
\end{proof}

The proof proceeds in the most simple manner if we focus on template edges that different populations and have no common children.

\begin{definition} \label{def:suitable}
An template edge $\TT_{1} \rightarrow \TT_{2}$ is \defterm{suitable} if

\begin{enumerate}
\item The parent and child do not contain the same population variables.
\item The parent and child have no common edge.
\end{enumerate}
\end{definition}

The next lemma shows that focusing on suitable edges incurs no loss of generality.

\begin{lemma} \label{lemma:suitable}
Suppose that a template BN contains an edge such that the parent and child do not contain the same population variables. Then the template BN contains a suitable edge. 
\end{lemma}

\begin{proof}
Suppose that there is an edge satisfying the population variable condition. Suppose that the parent and child share a common child. Since the edge satisfies the condition, the set of population variables in the common child differs from at least one of  $\TT_{1}, \TT_{2}$. Therefore there is another edge from one of  $\TT_{1} \rightarrow \TT_{2}$ as parent to a new child that satisfies the population variable condition. If this edge is not suitable, there must be another shared child. Repeating this argument, we eventually arrive at an edge satisfying the population variable condition  where the child node is a sink node without children. This edge is suitable.
\end{proof}

Consider a suitable template edge $\TT_{1} \rightarrow \TT_{2}$ that produces a bidirected ground edge $\FG{\TT_{1}} \leftrightarrow \FG{\TT_{2}}$. For simplicity we assume that $\TT_{1}$ and $\TT_{2}$ are binary variables with domain $\{\true,\false\}$. (This incurs no loss of generality as we can choose a database $\QC$ in which only two values occur.) Let $\Pa{\TT_{2}}$ be the parents of $\TT_{2}$ other than $\TT_{1}$. Since the template edge is not redundant \cite{Pearl1988}, there is a parent value setting $\Pa{\TT_{2}} = \parents$ such that $\TT_{1}$ and $\TT_{2}$ are conditionally dependent given $\Pa{\TT_{2}} = \parents$. This implies that the conditional distribution of $\TT_{1}$ is different for each of the two possible values of $\TT_{2}$:
 In terms of the template Bayesian network parameters, this implies that

\begin{equation} \label{eq:dependence}
\frac{\cprob{\TT_{2} = \false}{\TT_{1} = \false,\parents}}{\cprob{\TT_{2} = \true}{\TT_{1} = \false,\parents}} \neq \frac{\cprob{\TT_{2} = \false}{\TT_{1} = \true,\parents}}{\cprob{\TT_{2} = \true}{\TT_{1} = \true,\parents}}.
\end{equation}

Let $\QC$ denote an assignment of values to all ground nodes other than the target nodes $\FG{\TT_{1}}$ and $ \FG{\TT_{2}}$. We assume that the input query $\QC$ assigns different relevant family counts $N_{1}$ to $\FG{\TT_{1}}$ and $N_{2}$ to $\FG{\TT_{2}}$. This is possible according to Lemma~\ref{lemma:grounding}. 

\subsection{Lowd's Equation and Relevant Family Counts}
The log-linear equation~\ref{def:log-diff-freq-eq}, specifies the conditional distribution of each target node given $\QC$ and a value for the other target node. We keep the assignment $\QC$ fixed throughout, so for more compact notation, we abbreviate the conditional distributions as

$$\joint(\FG{{\TT_{1}}} = \TV_{1}| \FG{{\TT_{2}}} = \TV_{2}) \equiv P(\FG{{\TT_{1}}} = \TV_{1}|\FG{{\TT_{2}}} = \TV_{2},\QC)$$ 

and similarly for $P(\FG{{\TT_{1}}} = \TV_{1}|\FG{{\TT_{2}}} = \TV_{2},\QC)$.

On the assumption that the dependency network is consistent, there is a joint distribution over the target nodes conditional on the assignment that agrees with the conditional distribution:

$$\frac{\joint(\FG{{\TT_{1}}} = \TV_{1}, \FG{{\TT_{2}}} = \TV_{2})}{\joint(\FG{{\TT_{2}}} = \TV_{2})}= \joint(\FG{{\TT_{1}}} = \TV_{1}| \FG{{\TT_{2}})}$$
and also with the conditional $\joint(\FG{{\TT_{2}}} = \TV_{2}| \FG{{\TT_{1}}}=\TV_{1}).$

Lowd \cite{Lowd2012} pointed out that 
this joint distribution satisfies the equations

\begin{equation}  \frac{\joint(\false,\false)}{\joint(\true,\false)} \cdot \frac{\joint(\true,\false)}{\joint(\true,\true)}= \frac{\joint(\false,\false)}{\joint(\true,\true)} = \frac{\joint(\false,\false)}{\joint(\false,\true)} \cdot \frac{\joint(\false,\true)}{\joint(\true,\true)} \label{eq:lowd-joint}
\end{equation}

Since the ratio of joint probabilities is the same as the ratio of conditional probabilities for the same conditioning event, consistency entails the following constraint on conditional probabilities via Equation~\eqref{eq:lowd-joint}:

{\small
\begin{equation}
\frac{\joint(\FG{{\TT_{2}}}=\false|\FG{{\TT_{1}}}=\false)}{\joint(\FG{{\TT_{2}}} = \true| \FG{{\TT_{1}}}=\false)} \cdot \frac{\joint(\FG{{\TT_{1}}}=\false|\FG{{\TT}_{2}}=\true)}{\joint(\FG{{\TT_{1}}} = \true| \FG{{\TT_{2}}}=\true)} =\frac{\joint(\FG{{\TT_{1}}}=\false|\FG{{\TT_{2}}}=\false)}{\joint(\FG{{\TT_{1}}} = \true| \FG{{\TT_{2}}}=\false)} \cdot \frac{\joint(\FG{{\TT_{2}}}=\false|\FG{{\TT_{1}}}=\true)}{\joint(\FG{{\TT_{2}}} = \true| \FG{{\TT_{1}}}=\true)} \label{eq:lowd-conditional}
\end{equation}
}
We refer to Equation~\ref{eq:lowd-conditional} as {\em Lowd's equation}. 
The idea of our proof is to show that Lowd's equations are satisfied only if the relevant family counts for the target nodes are the same. According to the log-linear equation, each conditional probability is proportional to a product of BN parameters. The first step is to show that in Lowd's equation, all BN parameter terms cancel out except for those that are derived from the family that comprises $\FG{\TT_{1}}$ and their $\FG{\TT_{2}}$ and their common grounding. 
This may not hold in general, but can be proved provided that the edge $\TT_{1} \rightarrow \TT_{2}$ satisfies two conditions.
\begin{definition} \label{def:suitable}
An template edge $\TT_{1} \rightarrow \TT_{2}$ is \textbf{suitable} if
\begin{enumerate}
\item It is possible to find a grounding $\grounding$ for both parent and child, and an assignment $\QC$ to all other nodes, such that the relevant family count for the $\TT_{2}$ family differs for $\FG{\TT_{1}} = \grounding \TT_{1}$ 
and $\FG{\TT_{2}} = \grounding \TT_{2}$.
\item the parent and child have no common edge.
\end{enumerate}
\end{definition}

\begin{lemma} \label{lemma:decompose-cond} The conditional probabilities for the target nodes can be written as follows:
\begin{equation}
\Gprob{\FG{{\TT_{2}}} = \TV_{2}} {\FG{\TT_{1}} = \TV_{1},\QC} \propto \cprob{\TT_{2} = \TV_{2}}{\TT_{1} = \TV_{1},\parents}^{(N/N_{2}+M_{\TT_2=\TV_{2}}/N_{2})} \cdot \pi_{\TT_2=\TV_{2}} \label{eq:decompose-t2}
\end{equation}
where $M_{\TT_2=\TV_{2}}$ and $\pi_{\TT_2=\TV_{2}}$ depend only on $\TV_{2}$ and not on $\TV_{1}$ and
\begin{equation}
\Gprob{\FG{{\TT_{1}}} = \TV_{1}} {\FG{\TT_{2}} = \TV_{2},\QC} \propto \cprob{\TT_{2} = \TV_{2}}{\TT_{1} = \TV_{1},\parents}^{(N/N_{1}+M_{\TT_1=\TV_{1}}/N_{1})} \cdot \pi_{\TT_1=\TV_{1}} \label{eq:decompose-t1}
\end{equation} 
where $M_{\TT_1=\TV_{1}}$ and $\pi_{\TT_1=\TV_{1}}$ depend only on $\TV_{1}$ and not on $\TV_{2}$.
\end{lemma}



\begin{proof}
We start with target node $\FG{\TT_{2}}$. (1) The log-linear equation~\ref{def:log-diff-freq-eq} contains a term for the children of $\FG{\TT_{2}}$. Since $\TT_{1}$ and $\TT_{2}$ share no children, the corresponding conditional probabilities do not depend on the value of $\FG{\TT_{1}}$, but only on $\QC$ and $\TT_{2}$. Thus the product of the BN parameters can be denoted as $\pi_{\TT_2=\TV_{2}}$. (2) The only other term in the log-linear equation is for the family of $\FG{\TT_{2}}$. Since $\QC$ is suitable, the only instantiated groundings for the parents of $\FG{\TT_{2}}$ agree with the values $\parents$. These groundings can be divided into those that agree with $\FG{\TT_{1}}$ and those that do not. (3) The log-linear terms for the latter do not depend on the value $\TV_{1}$ of $\FG{\TT_{1}}$,  hence their number can be written as $M_{\TT_2=\TV_{2}}$. (4) For groundings that are consistent with both  $\FG{\TT_{1}}$ and $\FG{\TT_{2}}$, their number does not depend on the values of $\FG{\TT_{1}}$ or $\FG{\TT_{2}}$. It depends only on $\QC$. Let this number be $N$. 

Now consider target node $\FG{\TT_{1}}$. (1) The log-linear equation~\ref{def:log-diff-freq-eq} contains a term for the family of $\FG{\TT_{1}}$. Since $\TT_{1}$ is a parent of $\TT_{2}$, the acyclicity of the template BN entails that $\TT_{2}$ is not a parent of $\TT_{1}$. Therefore  the conditional probabilities for the family of $\FG{\TT_{1}}$ do not depend on the value of $\FG{\TT_{2}}$, but only on $\QC$ and $\TT_{1}$. (2) The log-linear equation~\ref{def:log-diff-freq-eq} also contains a term for the children of $\TT_{1}$ other than $\TT_{2}$. Since the edge $\TT_{1} \rightarrow \TT_{2}$ is suitable, the two nodes do not share a child, so these terms also do not depend on the value of $\FG{\TT_{2}}$. Thus collectively, the product of the terms (1) and (2) can be written as $\pi_{\TT_1=\TV_{1}} $. The remaining terms are groundings for $\FG{\TT_{1}}$ and the family of $\TT_{2}$. These groundings can be divided into those that agree with $\FG{\TT_{2}}$ and those that do not. (3) The log-linear terms for the latter do not depend on the value $\TV_{2}$ of $\FG{\TT_{2}}$,  hence their number can be written as $M_{\TT_1=\TV_{1}}$. (4) The number of groundings that are consistent with both  $\FG{\TT_{1}}$ and $\FG{\TT_{2}}$ is denoted by $N$ as above.
\end{proof}

\begin{lemma} \label{lemma:family-agree}
Suppose that conditions~\eqref{eq:decompose-t2} and~\eqref{eq:decompose-t1} of Lemma~\ref{lemma:decompose-cond} hold. Then Lowd's Equation~\eqref{eq:lowd-conditional} holds if and only if $N_{1} = N_{2}$. 
\end{lemma}

\begin{proof}
Observe that in Equation~\eqref{eq:lowd-conditional}, each term on the left has a corresponding term with the same value for the target node assignment and the opposing conditioning assignment. For instance, the term $\joint(\FG{{\TT_{2}}}=\false|\FG{{\TT_{1}}}=\false)$ on the left is matched with the term $\joint(\FG{{\TT_{2}}}=\false|\FG{{\TT_{1}}}=\true)$ on the right. This means that the products in the log-linear expression are the same on both sides of the equation except for those factors that depend on {\em both} $\TV_{1}$ and $\TV_{2}$. Continuing the example, the factors $$\cprob{\TT_{2} = \false}{\TT_{1} = \false,\parents}^{(M_{\false}/N_{2})} \cdot \pi_{\TT_2=\TV_{2}}$$ on the left equal the factors $$\cprob{\TT_{2} = \false}{\TT_{1} = \true,\parents}^{(M_{\TT_1=\TV_{1}}/N_{2})}\cdot \pi_{\TT_2=\TV_{2}}$$ on the right side of the equation. They therefore cancel out, leaving only the term $$\cprob{\TT_{2} = \false}{\TT_{1} = \false,\parents}^{N/N_{2}}$$ on the left and the term $$\cprob{\TT_{2} = \false}{\TT_{1} = \false,\parents}^{N/N_{2}}$$ on the right. Lowd's equation can therefore be reduced to an equivalent constraint with only such BN parameter terms. For further compactness we abbreviate such terms as follows
$$\cprob{\TV_{2}}{\TV_{1}} \equiv \cprob{\TT_{2} = \TV_{2}}{\TT_{1} = \TV_{1},\parents}.$$ With this abbreviation, the conditions of Lemma~\ref{lemma:decompose-cond} entail that Lowd's equation~\ref{eq:lowd-conditional} reduces to the equivalent expressions.
\begin{eqnarray}
\frac{\cprob{\false}{\false}^{N/N_{2}}}{\cprob{\true}{\false}^{N/N_{2}} }  \cdot \frac{\cprob{\true}{\false}^{N/N_{1}} }{\cprob{\true}{\true}^{N/N_{1}} }  & = & \frac{\cprob{\false}{\false}^{N/N_{1}} }{\cprob{\false}{\true}^{N/N_{1}} }  \cdot \frac{\cprob{\false}{\true}^{N/N_{2}} }{\cprob{\true}{\true}^{N/N_{2}} } \\
(\frac{\cprob{\false}{\false}}{\cprob{\true}{\false} })^{\left(N/N_{2}-N/N_{1}\right)}   & = &  (\frac{\cprob{\false}{\true} }{\cprob{\true}{\true}})^{\left(N/N_{2}-N/N_{1}\right)} \label{eq:transform-ratio}
\end{eqnarray}

By the nonredundancy  assumption~\eqref{eq:dependence} on the BN parameters, we have
$$\frac{\cprob{\false}{\false}}{\cprob{\true}{\false} }   \neq  \frac{\cprob{\false}{\true} }{\cprob{\true}{\true}}$$

so Equation~\ref{eq:transform-ratio} implies that 
$$N_{1} = N_{2}, $$ which establishes the lemma. 
\end{proof}

The main theorem now follows as follows: Lemma~\ref{lemma:grounding} entails that if the dependency network is consistent, the log-linear equations satisfy Lowd's equation with the bidirected ground edge $\FG{\TT_{1}} \leftrightarrow \FG{\TT_{2}}$ and the query conjunction $\QC$ that satisfies the BN non-redundancy condition. Lemmas~\ref{eq:lowd-conditional} and~\ref{lemma:suitable} show that if the template BN is relational, it must contain a suitable edge $\TT_{1} \rightarrow \TT_{2}$. Lemma~\ref{lemma:family-agree} 
 together with Lowd's equation entails that the relevant counts for $\FG{\TT_{1}}$ and $\FG{\TT_{2}}$ must then be the same. But the query conjunction $\QC$ was chosen so that the relevant counts are different. This contradiction shows that Lowd's equation is unsatisfiable, and therefore no joint distribution exists that is consistent with the BN conditional distributions specified by the log-linear Equation~\ref{def:log-diff-freq-eq}.
\bibliographystyle{plain}
\bibliography{master}
\end{document}